\documentclass[journal,onecolumn]{IEEEtran} 
\usepackage{fixltx2e,fix-cm}
\usepackage{amssymb}
\usepackage{amsmath}
\usepackage{graphicx}
\usepackage{makeidx}
\usepackage{multicol}
\usepackage{mathptmx}
\usepackage{ifthen}
\usepackage{url}	
\usepackage{cite}

\usepackage{subcaption}
\usepackage{paralist} 
\usepackage{xspace}

\usepackage{tikz}
\usepackage{listings}
\usepackage{color}
\usepackage{xcolor}
\usepackage[capitalise,english,nameinlink]{cleveref} 
\usepackage{pifont}
\usepackage{comment}
\usepackage{wrapfig}
\usepackage{stmaryrd}

\makeindex

\usepackage[per-mode=symbol]{siunitx}
\usepackage{mathtools} 
\usepackage{aliascnt}
\usepackage[appendix=strip]{apxproof}
\theoremstyle{plain}
\newtheoremrep{theorem}{Theorem}[section]
\Crefname{theorem}{Thm.}{Thm.}
\newaliascnt{propositioncnt}{theorem}
\newtheoremrep{proposition}[propositioncnt]{Proposition}
\Crefname{propositioncnt}{Prop.}{Prop.}
\newaliascnt{lemmacnt}{theorem}
\newtheoremrep{lemma}[lemmacnt]{Lemma}
\Crefname{lemmacnt}{Lem.}{Lem.}
\newaliascnt{corollarycnt}{theorem}
\newtheoremrep{corollary}[corollarycnt]{Corollary}
\Crefname{corollarycnt}{Cor.}{Cor.}
\newaliascnt{factcnt}{theorem}

\Crefname{factcnt}{Fact}{Fact}
\newaliascnt{assumptioncnt}{theorem}

\Crefname{assumptioncnt}{Asm.}{Asm.}
\theoremstyle{definition}
\newaliascnt{remarkcnt}{theorem}
\newtheorem{remark}[remarkcnt]{Remark}
\Crefname{remarkcnt}{Rem.}{Rem.}
\newaliascnt{notationcnt}{theorem}

\Crefname{notationcnt}{Notation}{Notation}
\newaliascnt{definitioncnt}{theorem}
\newtheorem{definition}[definitioncnt]{Definition}
\Crefname{definitioncnt}{Def.}{Def.}
\newaliascnt{requirementcnt}{theorem}
\newtheorem{requirement}[requirementcnt]{Requirement}
\Crefname{requirementcnt}{Requirement}{Requirements}

\usepackage[colorinlistoftodos,textsize=footnotesize]{todonotes}

\newcommand{\dfn}[1]{\index{#1}\emph{#1}}

\usepackage{caption}
\usepackage{subcaption}

\usepackage{gnuplot-lua-tikz}

\usepackage[acronym,shortcuts,nopostdot,nogroupskip,nomain,nonumberlist,hyperfirst=false]{glossaries}
\glsdisablehyper  

\newacronym{aadl}   {AADL}  {Architecture Analysis \& Design Language}
\newacronym{ads}    {ADS}   {automated driving system}
\newacronym{api}    {API}   {application programming interface}
\newacronym{can}  {CAN}  {controller area network}
\newacronym{cps}  {CPS}  {cyber-physical system}
\newacronym{ctl}    {CTL}   {computation tree logic}
\newacronym{dsl}    {DSL}   {domain-specific language}
\newacronym{dsm}   {DSM}  {domain-specific modelling}
\newacronym{dsml}   {DSML}  {domain-specific modelling language}
\newacronym{emf}    {EMF}   {Eclipse Modeling Framework}
\newacronym[longplural={finite state automata}]  
           {fsa}    {FSA}   {finite state automaton}
\newacronym{fsm}    {FSM}   {finite state machine}
\newacronym{gppl}   {GPPL}  {general-purpose programming language}
\newacronym[longplural={hybrid automata}]
           {ha}     {HA}    {hybrid automaton}
\newacronym{hp}     {HP} {hybrid program}
\newacronym{hs}     {HS}    {hybrid system}
\newacronym{ide}    {IDE}   {integrated development environment}
\newacronym{iot}    {IoT}   {Internet of Things}
\newacronym[longplural={linear hybrid automata}]
            {lha}     {LHA}    {linear hybrid automaton}
\newacronym{ltl}    {LTL}   {linear temporal logic}
\newacronym{mitl}   {MITL}  {metric interval temporal logic}
\newacronym{mtl}   {MTL}  {metric temporal logic}
\newacronym[longplural={finite state automata}]
           {nfa}    {NFA}   {nondeterministic finite automaton}
\newacronym{ode}    {ODE}   {ordinary differential equation}
\newacronym{omg}    {OMG}   {Object Management Group}
\newacronym{po}    {PO}   {proof obligation}
\newacronym{pov}    {POV}   {principal other vehicle}
\newacronym[longplural={parametric timed automata}]
        {pta}    {PTA}   {parametric timed automaton}
\newacronym[longplural={parametric timed data automata}]
        {ptda}    {PTDA}   {parametric timed data automaton}
\newacronym{rte}    {RTE}   {real-time and embedded}
\newacronym{smt}    {SMT}   {satisfiability modulo theories}
\newacronym{stl}    {STL}   {signal temporal logic}
\newacronym{sv}    {SV}   {subject vehicle}
\newacronym{sysml}  {SysML} {Systems Modeling Language}

\newacronym[longplural={timed automata}]
           {ta}     {TA}    {timed automaton}
\newacronym{tctl}   {TCTL}  {timed computation tree logic}
\newacronym[longplural={timed systems}]
           {ts}     {TS}    {timed system}
\newacronym[longplural={timed symbolic automata}]
           {tsa}     {TSA}    {timed symbolic automaton}
\newacronym[longplural={timed symbolic weighted automata}]
           {tswa}     {TSWA}    {timed symbolic weighted automaton}
\newacronym{ui}    {UI}   {user interface}
\newacronym{uml}    {UML}   {Unified Modeling Language}

\makeglossaries

\usetikzlibrary{arrows,automata,positioning}
\usetikzlibrary{patterns}
\usetikzlibrary{intersections,calc}
\usetikzlibrary{shapes,shapes.misc,shapes.multipart,shadows}
\usetikzlibrary{decorations.text}
\usetikzlibrary{decorations.markings}
\usetikzlibrary{decorations.pathmorphing}

\tikzstyle{every node}=[initial text=]
	\tikzstyle{location}=[circle, minimum size=12pt, draw=black, fill=blue!10, inner sep=1pt] 
\tikzstyle{final}=[double]
\tikzstyle{accepting}=[final]
\tikzstyle{PTPMOPT}=[,dashed,color=red,semithick]

\definecolor{coloract}{rgb}{0.50, 0.70, 0.30}
\definecolor{colorclock}{rgb}{0.4, 0.4, 1}
\definecolor{colorconst}{rgb}{0.50, 0.20, 0.00}
\definecolor{colordisc}{rgb}{1, 0, 1}
\definecolor{colorloc}{rgb}{0.4, 0.4, 0.65}
\definecolor{colorparam}{rgb}{1, 0.6, 0.0}
\definecolor{colorvar}{rgb}{0.6, 0.7, 1}
\definecolor{colorlvar}{rgb}{0.4, 0.4, .5}
\definecolor{colordparam}{rgb}{.9, 0.8, 0.0}

\definecolor{mygray}{rgb}{0.5,0.5,0.5}

\newcommand{\PR}{P}
\newcommand{\slsd}{\mathrm{slsd}}

\usepackage{rotating}

\usepackage{booktabs}
\usepackage{tabularx}

\usepackage{proof}
\usepackage{mathpartir}

\begin{document}

\title{Responsibility-Sensitive Safety\\
{\LARGE an Introduction 
with an Eye to Logical Foundations and Formalization
}}

\author{Ichiro Hasuo$^{1,2}$
\thanks{$^{1}$%
National Institute of Informatics,
Hitotsubashi 2-1-2,
Tokyo 101-8430, Japan.
        {i.hasuo@acm.org}}%
\thanks{$^{2}$
SOKENDAI (The Graduate University for Advanced Studies),  Hayama, Japan. 
}
}
\maketitle

\begin{abstract}
\emph{Responsibility-sensitive safety} (RSS) is an approach to the safety of automated driving systems (ADS). It aims to introduce mathematically formulated \emph{safety rules}, compliance with which guarantees collision avoidance as a mathematical theorem. However, despite the emphasis on mathematical and logical guarantees, the logical foundations and formalization of RSS are largely an unexplored topic of study. In this paper, we present an introduction to RSS, one that we expect will bridge between different research communities
and 
 pave the way to a logical theory of RSS, its mathematical formalization, and software tools of practical use.
\end{abstract}




 \section{What is RSS?}
\label{subsec:whatIsRSS}
\emph{Responsibility-sensitive safety} (RSS) is an approach to the safety of automated driving systems (ADS) that has been attracting growing attention. Safety is a central theme in every aspect of  \ac{ads} research. RSS is unique among these approaches in that it provides \emph{safety rules} that are rigorously formulated in mathematical terms. Unlike most algorithms and techniques studied for \ac{ads}, RSS is not so much about \emph{how to drive safely} as about \emph{breaking down the ultimate goal} (namely safety) \emph{into concrete and checkable conditions}. The goal of RSS is that those safety rules  guarantee  \ac{ads} safety in the rigorous form of \emph{mathematical proofs}. How far RSS has gone towards this ultimate goal of ADS safety proofs, and what is still needed for this ultimate goal, are questions that we will discuss later in Section~\ref{subsec:RSSFuture}. 
\index{responsibility-sensitive safety}\index{RSS|see{responsibility-sensitive safety}}

\setcounter{footnote}{2}
The methodology and basic techniques of RSS were first introduced by researchers at Mobileye/Intel~\cite{ShalevShwartzSS17RSS}. The safety rules derived in RSS have a variety of possible usages (as we will discuss in Section~\ref{subsec:RSSUsages}), giving RSS a multifaceted character. Moreover, its presentation to non-expert audience such as Mobileye/Intel's webpage\footnote{\url{https://www.mobileye.com/responsibility-sensitive-safety}} tends to emphasize RSS's ideological aspects. All these have somewhat  blurred the theoretical substance of RSS. 

In this introductory paper, in contrast, we aim at a theoretical account of RSS mainly from the viewpoints of  logic and software science. We expect that such an introductory account will bridge between different research communities
and 
 pave the way to a logical theory, formalization, and software tools---aspects of the RSS study that are largely unexplored at this moment. Specifically, in the current paper, we 1) formulate two requirements of safety rules (\cref{req:RSSSafetyAssurance,req:RSSResponsibility}), and 2) formulate their assume-guarantee relationship (see \cref{fig:RSSFrmwk}).

\begin{remark}[RSS terminologies]\label{rem:RSSTerminologies}
Although there are already a number of research works on RSS as well as practical works (including the efforts to use it in international standards),  comprehensive technical accounts on RSS are scarce. In particular, terminologies for the core concepts in RSS do not seem to be fixed in the literature.

In this paper, we will use the terminologies listed in Table~\ref{table:RSSTerminologies}. Note that they may differ from the terminologies used in the literature (which vary from a paper to another). The notions in the table will be introduced in due course. Their correspondence to those used in the original paper~\cite{ShalevShwartzSS17RSS} is also shown in the table. 
\end{remark}

\begin{table}[tp]
 \caption{RSS terminologies}
 \label{table:RSSTerminologies}
 \centering
\newcommand{\hdrule}{\midrule[\heavyrulewidth]}
 \begin{tabular}{p{3.5cm}p{3cm}p{5.5cm}}
  in this paper  & in~\cite{ShalevShwartzSS17RSS} & \\\hdrule
  RSS safety rule  & (no explicit name) & A pair of an RSS safety condition and a proper response \\\midrule
  RSS safety\newline condition & safe distance & An ``instantly-checkable'' condition that guarantees safety henceforth. Often stated in terms of a safety distance 
  \\\midrule
  proper response & proper response & A control strategy that realizes safety
  \\\midrule
  RSS responsibility principle & safety rule, \newline common sense  rule & A high-level and informal principle used as assumptions on the behaviors of vehicles
 \end{tabular}
\end{table}

\section{Elements of RSS}
\label{subsec:RSSElem}
\subsection{RSS Safety Rules}
\label{subsubsec:RSSSafetyRules}

The central constructs in RSS are \emph{RSS safety rules}. Here is their semi-formal definition---a concrete example will appear  in Section~\ref{subsec:RSSEx}. 

\begin{definition}[RSS safety rule, RSS safety condition, proper response]\label{def:RSSSafetyRule}
\index{RSS safety rule}
\index{RSS safety condition}
\index{proper response}
An \emph{RSS safety rule} is a pair $R=(C,\PR)$ of an RSS safety condition $C$ and a proper response $\PR$, designed  specifically for each driving scenario (straight road or crossing, other cars driving in the same direction or another, etc.). Here,
\begin{itemize}
 \item an \emph{RSS safety condition} $C$ is a rigorous condition formulated in mathematical terms,
which must be \emph{instantly checkable} (in the sense that it mentions the current values of physical quantities, not the future ones); and
 \item 
 a \emph{proper response} $\PR$ is a control strategy that achieves safe driving, typically until the vehicle comes to a halt.
\end{itemize}
\end{definition}

Here is the first fundamental requirement of an RSS safety rule. 
\begin{requirement}[safety assurance in RSS] 
\label{req:RSSSafetyAssurance}
\index{safety assurance (in RSS)}
 An RSS safety rule $R=(C,\PR)$ is required to satisfy the following:
\begin{quote}
\emph{  Whenever the RSS condition $C$ is satisfied, executing the proper response $\PR$ from that moment leads to safe driving (i.e.\ driving with no collision).}
\end{quote}
\end{requirement}
By considering the beginning  (at time $0$) of the execution of $\PR$, the above requirement implies, in particular, that
\begin{quote}
 \emph{Whenever the RSS condition $C$ is satisfied, there is no collision at that moment.}
\end{quote}

Ideally, this safety assurance should be mathematically proved, building on the rigorous formulation of the RSS safety condition $C$ as well as that of the proper response $\PR$. Such a mathematical proof makes the requirement a \emph{theorem}.

A crucial feature of an RSS safety rule $R=(C,\PR)$ is that it reduces the problem of safety \emph{in the future} to a condition that can be verified \emph{at present},  namely the RSS safety condition $C$. The RSS safety condition $C$ ensures the safety of executions of the proper response $\PR$, much like a \emph{precondition} in a program logic ensures the safety of executions of a program  (see e.g.~\cite{Winskel93}). 

Such reduction of the future to the present is not easy. In fact, ensuring future safety is impossible without posing suitable assumptions about the behaviors of other vehicles. Figure~\ref{fig:rssNoResponsibility} is a situation that is often used as an example in the literature (including~\cite{ShalevShwartzSS17RSS}): a collision happens if the other vehicles come close to the \ac{sv}; and the \ac{sv} can do nothing to avoid it.

\begin{figure}[tbp]
 \centering
 \includegraphics[width=4cm]{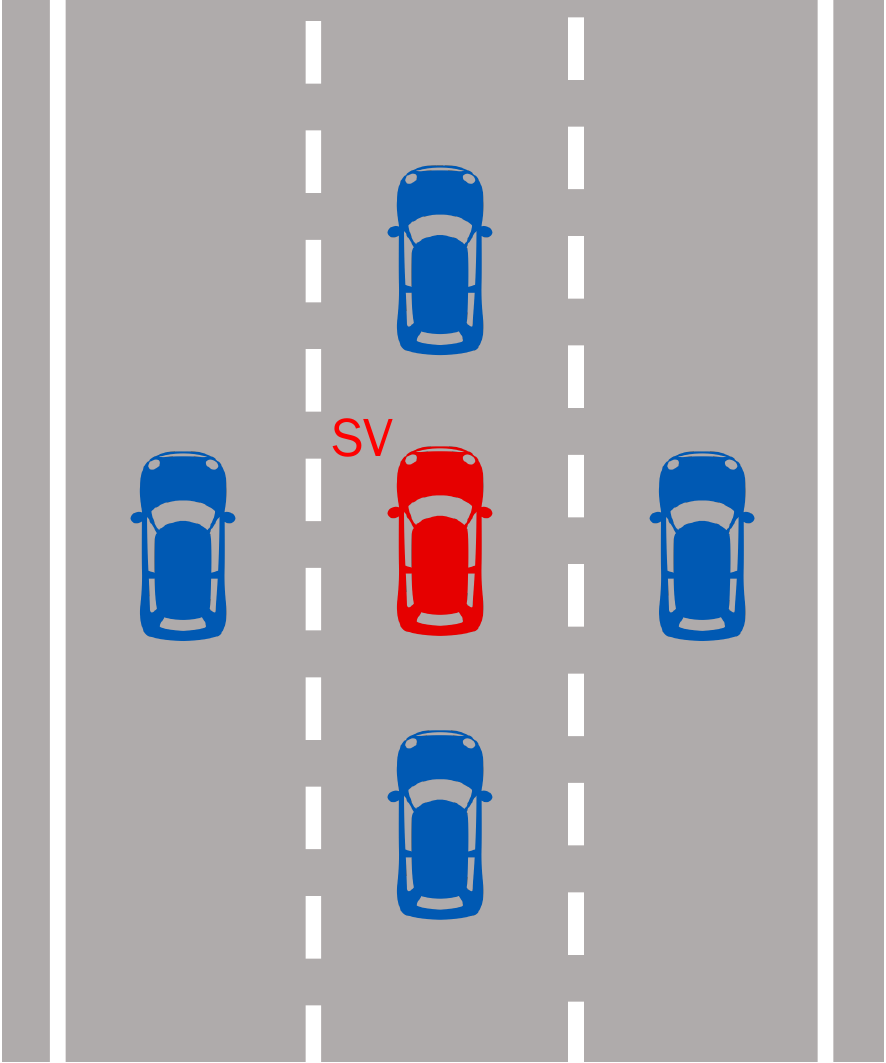}
 \caption{the \acf{sv} is surrounded by other vehicles. When they come close, there is nothing that the \ac{sv} can do to avoid a collision.}
\label{fig:rssNoResponsibility}
\end{figure}

\subsection{RSS Responsibility Principles}
\label{subsubsec:RSSPrinciples}
It is therefore needed to formulate certain behavioral constraints that all traffic participants are expected to respect. Such constraints can be thought of as \emph{contracts} in driving situations, when the latter are identified with multi-agent systems. Each traffic participant must respect these constraints; moreover, each agent can act assuming that all the other agents obey these constraints. In RSS, these behavioral constraints are derived from a high-level idea of \emph{responsibility}. 

Specifically, RSS expresses its idea of responsibility in the following informal form of \emph{RSS responsibility principles}. (Note that the name ``RSS responsibility principle'' differs from those which appear in the literature---see Remark~\ref{rem:RSSTerminologies}.)

\begin{definition}[RSS responsibility principles~\cite{ShalevShwartzSS17RSS}]
\label{def:RSSPrinciples}
\index{RSS responsibility principle}
\hfill
\begin{enumerate}
 	 \item Don't hit the car in front of you
	 \item 
Don't cut in recklessly
	 \item 
Right of way is given, not taken
	 \item 
Be cautious in areas with limited visibility
	 \item 
If you can avoid a crash without causing another one, you must

\end{enumerate}
\end{definition}

Note that these principles indeed embody a natural idea of responsibility, or ``duties of care,'' in driving. They are close to those rules commonly expected of human drivers.

The uses of the five RSS responsibility principles are twofold. Firstly, they can be used to establish Requirement~\ref{req:RSSSafetyAssurance}---more specifically, in a mathematical proof of the satisfaction of Requirement~\ref{req:RSSSafetyAssurance}, one can 
derive some  assumptions from the informal principles in Definition~\ref{def:RSSPrinciples}, and require them of the other vehicles' behaviors. 

The second use of the RSS responsibility principles is that, as a participant in a traffic scenario (that is thought of as a multi-agent system), the \acf{sv} must itself respect those principles:

\begin{requirement}[responsibility in RSS] 
\label{req:RSSResponsibility}
\index{responsibility (in RSS)}
 An RSS safety rule $R=(C,\PR)$ is required to satisfy the following:
\begin{quote}
 Let $E$ be an arbitrary execution of the proper response $\PR$; assume that the execution $E$ starts at a state in which the RSS safety condition $C$ is true. Then the execution $E$ respects the RSS responsibility principles (Definition~\ref{def:RSSPrinciples}).
\end{quote}
\end{requirement}

 Requirement~\ref{req:RSSResponsibility} is expected to be mathematically proven, too. In doing so, one has to turn the RSS responsibility principles (that are only informally stated, Definition~\ref{def:RSSPrinciples}) into certain rigorous mathematical conditions, in a way specific to the driving scenario in question. See Section~\ref{subsec:RSSEx} for an example.

\section{The RSS Framework}\label{subsec:RSSFrmwk}
The elements of RSS that we have described constitute what we call the \emph{RSS framework}, depicted in Figure~\ref{fig:RSSFrmwk}. Each RSS safety rule $R=(C,\PR)$ is formulated for an individual driving scenario (``single-lane same-direction,'' ``cut in,'' ``crossing,'' etc.)---the design of the rule, as well as the proofs that it satisfies Requirements~\ref{req:RSSSafetyAssurance}--\ref{req:RSSResponsibility}, is heavily dependent on the choice of a driving scenario. An example will be presented in Section~\ref{subsec:RSSEx}.

\begin{figure}[tbp]
\includegraphics[width=\textwidth]{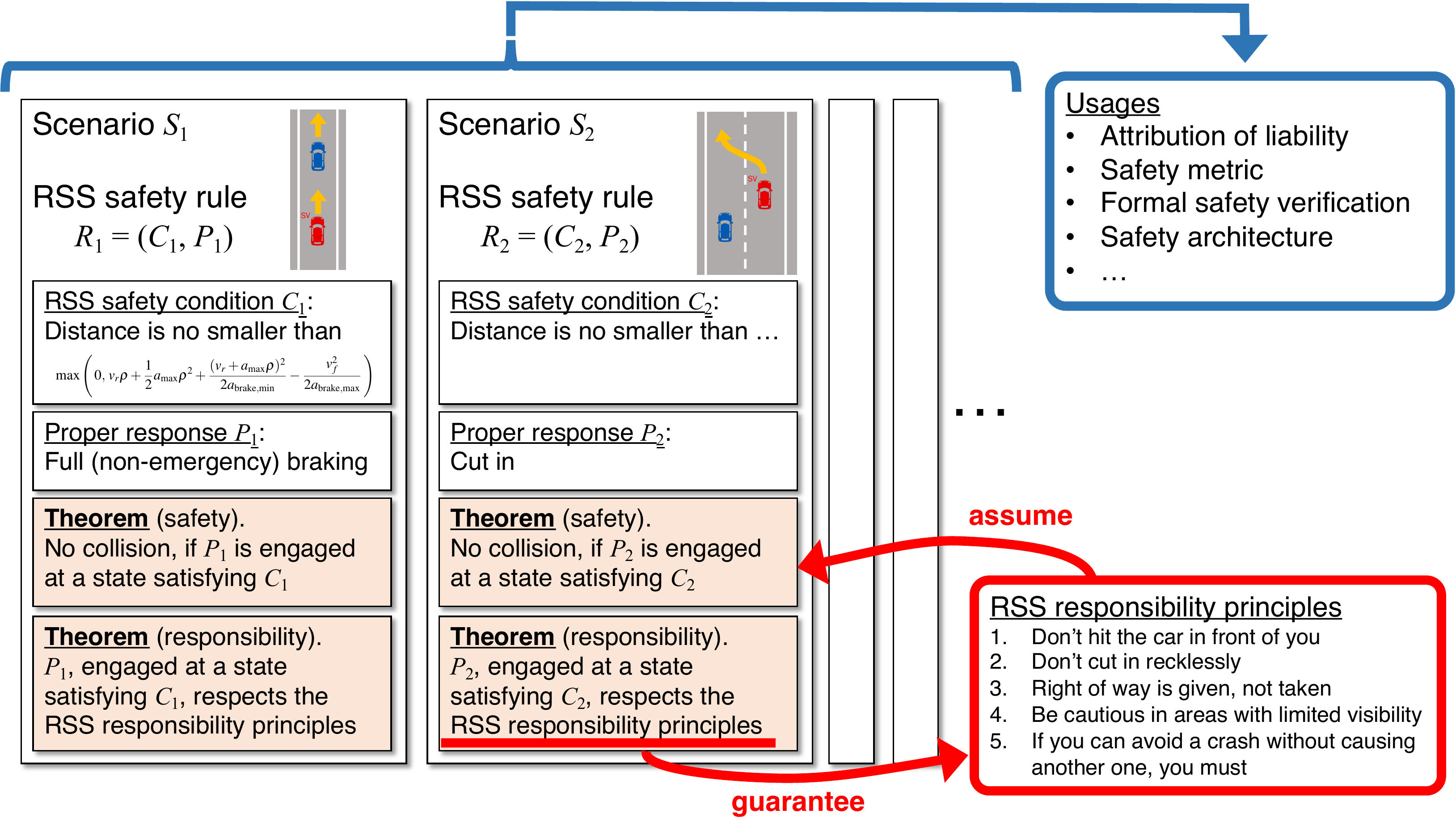}
 \caption{the RSS framework. An RSS safety rule is formulated for an individual driving scenario. Each rule is expected to satisfy safety (Requirement~\ref{req:RSSSafetyAssurance}) and responsibility (Requirement~\ref{req:RSSResponsibility}). The latter \emph{guarantees} that the subject vehicle acts responsibly; conversely, the safety proof assumes that the other vehicles act responsibly. }
 \label{fig:RSSFrmwk}
\end{figure}

In the proof of safety, it is assumed that other vehicles respect the RSS responsibility principles (Definition~\ref{def:RSSPrinciples})---otherwise safety is often unachievable, see Figure~\ref{fig:rssNoResponsibility}. Conversely, it is needed to show that the \ac{sv} respects the RSS responsibility principles, so that other vehicles can rely on it when they plan their behaviors. This is Requirement~\ref{req:RSSResponsibility} (responsibility).

While concrete usages of the RSS framework are discussed later in Section~\ref{subsec:RSSUsages}, we can already see that conceptual benefits of RSS are significant. 

For one, the RSS safety rules give a precise recipe for safe driving, in which 1) triggering conditions and control strategies are specified in mathematical terms, and 2) safety is established through a mathematical proof. RSS can therefore be seen as a promising approach to the goal of \emph{formal verification of ADS safety}. 

It is, however, unrealistic to expect that RSS realizes a world with zero traffic accidents. Even if the \ac{sv} acts according to the RSS safety rules, other vehicles may not, especially those driven by human drivers. RSS is still useful in such situations with traffic accidents. The RSS responsibility principles explicate natural ``contracts'' that each traffic participant is expected to follow. These principles can therefore be used to identify who is liable for an accident---namely the one who did not respect them. 

The last argument can be turned upside down and yield the following: 
\begin{quote}
 \emph{a vehicle is not responsible for an accident as long as it respects the RSS responsibility principles}. 
\end{quote}
This is an answer to a major challenge that is currently hindering  large-scale deployment of automated driving, namely
\begin{quote}
\emph{the difficulty of determining the boundary of the responsibilities of automated driving systems and their manufacturers.
}\end{quote}

The above boundary of responsibilities is currently vague for \ac{ads}, which exposes the \ac{ads} manufacturers to the risk of unexpected and exceeding liabilities. This risk is a big burden when a company wants to run an \ac{ads} business, potentially blocking the development and deployment of the   \ac{ads} technology. Clarification of the boundary of responsibilities is therefore pursued by many parties, including standardization bodies such as ISO, UL, and SAE. Indeed, the use of RSS is often advocated in these standardization efforts,   among which the efforts towards the IEEE 2846 standard are particularly notable.\footnote{\url{https://sagroups.ieee.org/2846/}}

\section{Example: an RSS Safety Rule for the Single-Lane Same-Direction Scenario}
\label{subsec:RSSEx}
We exhibit an example of an RSS safety rule and  its safety and responsibility proofs. The example is one of the first examples in the literature and is taken from~\cite{ShalevShwartzSS17RSS}.

\subsection{The Scenario $S_{\slsd}$}
The driving scenario in question, denoted by $S_{\slsd}$, has a single lane and two vehicles driving in the same direction. See Figure~\ref{fig:singleLaneSameDir}: the vehicle behind is the one under our control (the \acf{sv}); the \ac{sv} drives behind another vehicle that is called the \emph{\acf{pov}}.

\begin{figure}[tbp]
\centering
\includegraphics[width=2cm]{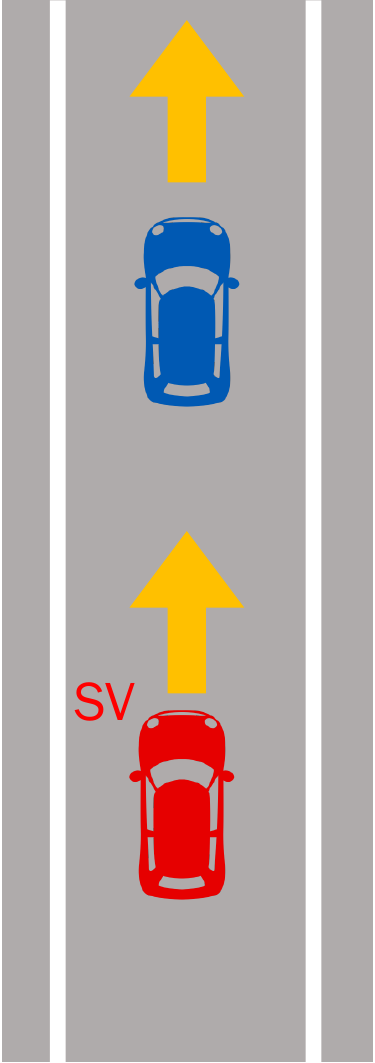}
 \caption{the single-lane same-direction scenario. The \acf{sv} is following the other vehicle. The latter is called the \emph{\acf{pov}}. }
 \label{fig:singleLaneSameDir}
\end{figure}

The goal is to avoid collision regardless of the behavior of the \ac{pov}. Here, however, we rule out some unrealistic but physically possible behaviors of the \ac{pov} from our consideration---such as the \ac{pov} being hit by a comet and suddenly coming to a halt.\footnote{The comet here is an example of the legal notion of \emph{act of God}. \url{https://en.wikipedia.org/wiki/Act_of_God}} The worst case within our consideration is the \ac{pov} engaging  emergency braking and coming to a halt in a short moment. We want the \ac{sv} to stop without colliding with the \ac{pov}. We want the \ac{sv} to do so by comfortable braking and not by emergency braking. Moreover, we have to take into account the \emph{response time}, the time between the moment the \ac{pov} starts braking and the moment the \ac{sv} starts its response. All these require a suitable distance between the \ac{pov} and the \ac{sv}; the question is how large exactly the distance should be.

Note that the above restriction of the scenario's scope (``no comet'') is similar to the use of the RSS responsibility principles, although there is no explicit RSS principle that corresponds to the ``no comet'' assumption. The argument here is ``if a collision happens because of a comet, the \ac{sv} is not held responsible,'' which is logically parallel to one that uses RSS responsibility principles such as ``if a collision happens because of another vehicle's reckless cut in, the \ac{sv} is not held responsible.''

\subsection{The RSS Safety Condition and the Proper Response}\label{subsubsec:RSSExCondPR}
In~\cite{ShalevShwartzSS17RSS}, an RSS safety condition $C_{\slsd}$ and a proper response $\PR_{\slsd}$ are given as follows. They together constitute an RSS safety rule $R_{\slsd}=(C_{\slsd},\PR_{\slsd})$.

 \subsubsection{The RSS Safety Condition  $C_{\slsd}$}
The condition  $C_{\slsd}$ is given by
        \begin{equation}\label{eq:RSSOnewayTraffic}
          \begin{aligned}
          & 	x_{f} - x_{r} 
	   \;>\;
          \max\left(\,0,\,
              v_{r}\rho + \frac{1}{2}a_{\mathrm{max}}  \rho^2 + \frac{(v_{r} + a_{\mathrm{max}}  \rho)^2}{2a_{\mathrm{brake,min}}} -\frac{v_{f}^2}{2a_{\mathrm{brake,max}}}\,\right).
          \end{aligned}
        \end{equation}
        Here, the following are dynamic parameters that describe the current state of the driving situation:
	    \begin{itemize}
	     \item  $x_{f}, x_{r}$ are the positions of the front vehicle (\ac{pov}) and the rear vehicle (\ac{sv}), respectively\footnote{We ignore the lengths of the cars for simplicity. }; and
	     \item $v_{f}, v_{r}$ are their velocities, respectively, modeled
        in the 1-dimensional lane coordinate.
		   
	    \end{itemize}
	    Besides, the following are static parameters for the driving scenario---they do not change from one state to another. Their values are decided  according to  traffic laws, regional customs,  vehicle specs, etc.:
	    \begin{itemize}
	     \item $\rho$ is the maximum \emph{response time} that
        the rear vehicle might take to initiate the required
        braking;
	     \item $a_{\max}$ is the maximum (forward) acceleration rate of the rear vehicle;

	     \item $a_\mathrm{brake,min}$ is the maximum
        comfortable braking rate for the rear vehicle; and
	     \item         $a_\mathrm{brake,max}$ is the maximum emergency braking rate for the front vehicle. We assume $0<a_\mathrm{brake,min}<a_\mathrm{brake,max}$.
	    \end{itemize}
 An example of these parameter values is found in~\cite{XuWW20TRB}: $\rho = 0.3$ s, $a_{\max} = 2$ \si{m\per\second^{2}}, $a_\mathrm{brake,min} = 4$ \si{m\per\second^{2}}, and $a_\mathrm{brake,max} =8$ \si{m\per\second^{2}}.

\subsubsection{The Proper Response  $\PR_{\slsd}$}
A proper response is a control strategy that is expected to \emph{avoid any collision in the future}, \emph{no matter how other vehicles would behave} (within a prescribed range of possible behaviors---recall the ``no comet'' argument in the above). The proper response  $\PR_{\slsd}$ in~\cite{ShalevShwartzSS17RSS} is 
\begin{quote}
 to engage the maximum comfortable braking, after a response time that is no bigger than $\rho$. The behavior during the response time is arbitrary.
\end{quote}
The response time is included since otherwise the control strategy would be unrealizable. Note that there are in fact many constraints on  the behavior during the response time: some are physical (an instant acceleration to the speed of light is ruled out, for example); others come from traffic laws, the design of a car, etc. Among all the possible behaviors during the response time, the worst case behavior is  accelerating at the maximum rate $a_{\max}$.

We note that proper responses in RSS---such as $\PR_{\slsd}$ in the above---may be quite harsh, in the sense that they are undesirable 
in view of other quality metrics than safety (such as comfort and fuel efficiency). We will discuss, later in  Section~\ref{subsubsec:RSSSafetyArch}, how frequent deployment of proper responses can be  avoided.

\subsection{The Safety Proof}
\begin{theorem}[safety of $R_{\slsd}$]\label{thm:RSSExSafetyThm}
 Let $s$ be a state of the driving scenario $S_{\slsd}$, and let the positions and the velocities of the vehicles in $s$ be denoted by $x_{f}, x_{r}, v_{f}, v_{r}$ as described in Section~\ref{subsubsec:RSSExCondPR}. 

Assume that the state $s$ satisfies the RSS safety condition $C_{\slsd}$. Consider an execution $e$ of the proper response $\PR_{\slsd}$ that starts at $s$ and ends when the \ac{sv} comes to a halt or collides into the \ac{pov}. Then, no collision occurs in this execution $e$ of $\PR_{\slsd}$.
\end{theorem}
The proof here follows the outline of the one presented in~\cite{ShalevShwartzSS17RSS}. 
\begin{proof}
In the execution $e$ of $\PR_{\slsd}$ in question, there are two arbitrary components: the behavior of the \ac{sv} during the response time, and the behavior of the \ac{pov}.  It is however obvious that the worst choices in terms of safety are 1) the \ac{sv} accelerates at the maximum rate $a_{\max}$, at the beginning of $e$, for the longest possible response time $\rho$, and 2) the \ac{pov}  engages the maximum braking, namely at the rate $a_\mathrm{brake,max}$, for the whole period of $e$. In the rest of the proof, we assume these behaviors of the \ac{sv} and the \ac{pov} 
 without  loss of generality.

We study the sign of the relative velocity of the \ac{sv} with respect to the \ac{pov}. To do so, we list up the possible patterns of the time-varying relationship between the velocities of the \ac{sv} and the \ac{pov}. See Figure~\ref{fig:RSSExVelocityPatterns}.
\begin{itemize}
 \item In Case~1, the relative velocity of \ac{sv} is always positive, therefore the \emph{inter-vehicle distance} between the \ac{sv} and the \ac{pov} becomes minimum when the \ac{sv} comes to a halt.
 \item In Cases~2--3, the relative velocity of \ac{sv} turns from negative to positive. Therefore the inter-vehicle distance becomes minimum either at the start (at time $0$) or when the \ac{sv} comes to a halt.
 \item In Case~4, the relative velocity of \ac{sv} is always negative, therefore the  inter-vehicle distance becomes minimum at the start (at time $0$). 
\end{itemize}
It is easy to see that the four cases are exhaustive (the assumption  $0<a_\mathrm{brake,min}<a_\mathrm{brake,max}$ is crucial here). It therefore suffices to ensure 
$x_{f} > x_{r}$
both at time $0$  and at the time when the \ac{sv} comes to a halt.

A necessary and sufficient condition for avoiding collision at time $0$ is, obviously,
\begin{equation}\label{eq:RSSExBeginning}
  x_{f}-x_{r}>0. 
\end{equation}

A condition for avoiding collision at the time when the \ac{sv} comes to a halt is derived as follows.
\begin{itemize}
 \item The \ac{sv} travels the distance $v_{r}\rho+\frac{1}{2}a_{\mathrm{max}}  \rho^2$ during the response time, at the end of which its velocity reaches $v_{r}+ a_{\mathrm{max}}\rho$.  The subsequent braking phase takes the time
 $\frac{v_{r} + a_{\mathrm{max}}  \rho}{a_\mathrm{brake,min}}$, during which the \ac{sv} travels the distance
 \begin{math}
  \frac{1}{2}(v_{r} + a_{\mathrm{max}}  \rho)\frac{v_{r} + a_{\mathrm{max}}  \rho}{a_\mathrm{brake,min}}
 \end{math}.
 \item To compute the distance that the \ac{pov} travels until the \ac{sv} comes to a halt, note first that we can restrict to Cases~1--3 in Figure~\ref{fig:RSSExVelocityPatterns}---this is because the minimum inter-vehicle distance is at time $0$ (not when the SV comes to a halt) in Case~4. We can easily see that, in Cases~1--3, the \ac{pov} comes to a halt earlier than the \ac{sv} does. Therefore the traveled distance for \ac{pov} (until the \ac{sv} comes to a halt) is $\frac{1}{2}v_{f}\frac{v_{f}}{a_\mathrm{brake,max}}$.
 \item We conclude that, in Cases~1--3 (Case~4 can be ignored as argued in the above), the positions of the SV and the POV when the SV comes to a halt are 
\begin{displaymath}
 x_{r}+ v_{r}\rho+\frac{1}{2}a_{\mathrm{max}}  \rho^2 +\frac{(v_{r} + a_{\mathrm{max}}  \rho)^2}{2a_{\mathrm{brake,min}}}
 \quad\text{and}\quad
 x_{f} + \frac{v_{f}^2}{2a_{\mathrm{brake,max}}},
\end{displaymath}
respectively. There is no collision at that moment if and only if 
\begin{equation}\label{eq:RSSExEnd}
  x_{f} - x_{r} > v_{r}\rho+\frac{1}{2}a_{\mathrm{max}}  \rho^2+\frac{(v_{r} + a_{\mathrm{max}}  \rho)^2}{2a_{\mathrm{brake,min}}} -  \frac{v_{f}^2}{2a_{\mathrm{brake,max}}}.
\end{equation}
\end{itemize}

The RSS safety condition (\ref{eq:RSSOnewayTraffic}) implies both (\ref{eq:RSSExBeginning}--\ref{eq:RSSExEnd}), and thus ensures that there is no collision at the beginning or at the end of an execution of $\PR_{\slsd}$, which in turn ensures that there is no collision at \emph{any} moment during the execution. This concludes the proof.
\end{proof}

\begin{figure}[tbp]
\centering

  \begin{subfigure}[b]{0.21\textwidth}
  \centering
 \scalebox{.7}{\begin{tikzpicture}[gnuplot]
\tikzset{every node/.append style={font={\normalsize}}}
\path (0.000,0.000) rectangle (6.604,4.572);
\gpcolor{color=gp lt color border}
\gpsetlinetype{gp lt border}
\gpsetdashtype{gp dt solid}
\gpsetlinewidth{1.00}
\draw[gp path] (0.952,0.985)--(1.132,0.985);
\draw[gp path] (6.051,0.985)--(5.871,0.985);
\node[gp node right] at (0.768,0.985) {0};
\draw[gp path] (1.501,0.985)--(1.501,1.165);
\draw[gp path] (1.501,4.263)--(1.501,4.083);
\node[gp node center] at (1.501,0.677) {$\rho$};
\draw[gp path] (0.952,4.263)--(0.952,0.985)--(6.051,0.985)--(6.051,4.263)--cycle;
\gpsetdashtype{gp dt 2}
\draw[gp path](1.502,3.324)--(1.502,0.986);
\node[gp node center,rotate=-270] at (0.292,2.624) {velocity};
\node[gp node center] at (3.501,0.215) {time};
\node[gp node right] at (4.583,3.929) {SV};
\gpcolor{rgb color={0.580,0.000,0.827}}
\gpsetdashtype{gp dt 1}
\gpsetlinewidth{3.00}
\draw[gp path] (4.767,3.929)--(5.683,3.929);
\draw[gp path] (0.952,3.170)--(1.004,3.185)--(1.055,3.199)--(1.107,3.213)--(1.158,3.228)%
  --(1.210,3.242)--(1.261,3.256)--(1.313,3.271)--(1.364,3.285)--(1.416,3.299)--(1.467,3.314)%
  --(1.519,3.314)--(1.570,3.285)--(1.622,3.256)--(1.673,3.228)--(1.725,3.199)--(1.776,3.170)%
  --(1.828,3.141)--(1.879,3.113)--(1.931,3.084)--(1.982,3.055)--(2.034,3.027)--(2.085,2.998)%
  --(2.137,2.969)--(2.188,2.941)--(2.240,2.912)--(2.291,2.883)--(2.343,2.854)--(2.394,2.826)%
  --(2.446,2.797)--(2.497,2.768)--(2.549,2.740)--(2.600,2.711)--(2.652,2.682)--(2.703,2.654)%
  --(2.755,2.625)--(2.806,2.596)--(2.858,2.567)--(2.909,2.539)--(2.961,2.510)--(3.012,2.481)%
  --(3.064,2.453)--(3.115,2.424)--(3.167,2.395)--(3.218,2.367)--(3.270,2.338)--(3.321,2.309)%
  --(3.373,2.281)--(3.424,2.252)--(3.476,2.223)--(3.527,2.194)--(3.579,2.166)--(3.630,2.137)%
  --(3.682,2.108)--(3.733,2.080)--(3.785,2.051)--(3.836,2.022)--(3.888,1.994)--(3.939,1.965)%
  --(3.991,1.936)--(4.042,1.907)--(4.094,1.879)--(4.145,1.850)--(4.197,1.821)--(4.248,1.793)%
  --(4.300,1.764)--(4.351,1.735)--(4.403,1.707)--(4.454,1.678)--(4.506,1.649)--(4.557,1.621)%
  --(4.609,1.592)--(4.660,1.563)--(4.712,1.534)--(4.763,1.506)--(4.815,1.477)--(4.866,1.448)%
  --(4.918,1.420)--(4.969,1.391)--(5.021,1.362)--(5.072,1.334)--(5.124,1.305)--(5.175,1.276)%
  --(5.227,1.247)--(5.278,1.219)--(5.330,1.190)--(5.381,1.161)--(5.433,1.133)--(5.484,1.104)%
  --(5.536,1.075)--(5.587,1.047)--(5.639,1.018)--(5.690,0.989)--(5.697,0.985);
\gpcolor{color=gp lt color border}
\node[gp node right] at (4.583,3.621) {POV};
\gpcolor{rgb color={0.000,0.620,0.451}}
\draw[gp path] (4.767,3.621)--(5.683,3.621);
\draw[gp path] (0.952,2.952)--(1.004,2.894)--(1.055,2.837)--(1.107,2.780)--(1.158,2.722)%
  --(1.210,2.665)--(1.261,2.607)--(1.313,2.550)--(1.364,2.493)--(1.416,2.435)--(1.467,2.378)%
  --(1.519,2.320)--(1.570,2.263)--(1.622,2.206)--(1.673,2.148)--(1.725,2.091)--(1.776,2.034)%
  --(1.828,1.976)--(1.879,1.919)--(1.931,1.861)--(1.982,1.804)--(2.034,1.747)--(2.085,1.689)%
  --(2.137,1.632)--(2.188,1.574)--(2.240,1.517)--(2.291,1.460)--(2.343,1.402)--(2.394,1.345)%
  --(2.446,1.287)--(2.497,1.230)--(2.549,1.173)--(2.600,1.115)--(2.652,1.058)--(2.703,1.000)%
  --(2.716,0.985);
\gpcolor{color=gp lt color border}
\gpsetdashtype{gp dt solid}
\gpsetlinewidth{1.00}
\draw[gp path] (0.952,4.263)--(0.952,0.985)--(6.051,0.985)--(6.051,4.263)--cycle;
\gpdefrectangularnode{gp plot 1}{\pgfpoint{0.952cm}{0.985cm}}{\pgfpoint{6.051cm}{4.263cm}}
\end{tikzpicture}
}
 \caption{Case 1}
   \end{subfigure}
  \qquad
  \begin{subfigure}[b]{0.21\textwidth}
  \centering
  \scalebox{.7}{\begin{tikzpicture}[gnuplot]
\tikzset{every node/.append style={font={\normalsize}}}
\path (0.000,0.000) rectangle (6.604,4.572);
\gpcolor{color=gp lt color border}
\gpsetlinetype{gp lt border}
\gpsetdashtype{gp dt solid}
\gpsetlinewidth{1.00}
\draw[gp path] (0.952,0.985)--(1.132,0.985);
\draw[gp path] (6.051,0.985)--(5.871,0.985);
\node[gp node right] at (0.768,0.985) {0};
\draw[gp path] (1.501,0.985)--(1.501,1.165);
\draw[gp path] (1.501,4.263)--(1.501,4.083);
\node[gp node center] at (1.501,0.677) {$\rho$};
\draw[gp path] (0.952,4.263)--(0.952,0.985)--(6.051,0.985)--(6.051,4.263)--cycle;
\gpsetdashtype{gp dt 2}
\draw[gp path](1.502,3.324)--(1.502,0.986);
\node[gp node center,rotate=-270] at (0.292,2.624) {velocity};
\node[gp node center] at (3.501,0.215) {time};
\node[gp node right] at (4.583,3.929) {SV};
\gpcolor{rgb color={0.580,0.000,0.827}}
\gpsetdashtype{gp dt 1}
\gpsetlinewidth{3.00}
\draw[gp path] (4.767,3.929)--(5.683,3.929);
\draw[gp path] (0.952,3.170)--(1.004,3.185)--(1.055,3.199)--(1.107,3.213)--(1.158,3.228)%
  --(1.210,3.242)--(1.261,3.256)--(1.313,3.271)--(1.364,3.285)--(1.416,3.299)--(1.467,3.314)%
  --(1.519,3.314)--(1.570,3.285)--(1.622,3.256)--(1.673,3.228)--(1.725,3.199)--(1.776,3.170)%
  --(1.828,3.141)--(1.879,3.113)--(1.931,3.084)--(1.982,3.055)--(2.034,3.027)--(2.085,2.998)%
  --(2.137,2.969)--(2.188,2.941)--(2.240,2.912)--(2.291,2.883)--(2.343,2.854)--(2.394,2.826)%
  --(2.446,2.797)--(2.497,2.768)--(2.549,2.740)--(2.600,2.711)--(2.652,2.682)--(2.703,2.654)%
  --(2.755,2.625)--(2.806,2.596)--(2.858,2.567)--(2.909,2.539)--(2.961,2.510)--(3.012,2.481)%
  --(3.064,2.453)--(3.115,2.424)--(3.167,2.395)--(3.218,2.367)--(3.270,2.338)--(3.321,2.309)%
  --(3.373,2.281)--(3.424,2.252)--(3.476,2.223)--(3.527,2.194)--(3.579,2.166)--(3.630,2.137)%
  --(3.682,2.108)--(3.733,2.080)--(3.785,2.051)--(3.836,2.022)--(3.888,1.994)--(3.939,1.965)%
  --(3.991,1.936)--(4.042,1.907)--(4.094,1.879)--(4.145,1.850)--(4.197,1.821)--(4.248,1.793)%
  --(4.300,1.764)--(4.351,1.735)--(4.403,1.707)--(4.454,1.678)--(4.506,1.649)--(4.557,1.621)%
  --(4.609,1.592)--(4.660,1.563)--(4.712,1.534)--(4.763,1.506)--(4.815,1.477)--(4.866,1.448)%
  --(4.918,1.420)--(4.969,1.391)--(5.021,1.362)--(5.072,1.334)--(5.124,1.305)--(5.175,1.276)%
  --(5.227,1.247)--(5.278,1.219)--(5.330,1.190)--(5.381,1.161)--(5.433,1.133)--(5.484,1.104)%
  --(5.536,1.075)--(5.587,1.047)--(5.639,1.018)--(5.690,0.989)--(5.697,0.985);
\gpcolor{color=gp lt color border}
\node[gp node right] at (4.583,3.621) {POV};
\gpcolor{rgb color={0.000,0.620,0.451}}
\draw[gp path] (4.767,3.621)--(5.683,3.621);
\draw[gp path] (0.952,3.498)--(1.004,3.441)--(1.055,3.383)--(1.107,3.326)--(1.158,3.269)%
  --(1.210,3.211)--(1.261,3.154)--(1.313,3.096)--(1.364,3.039)--(1.416,2.982)--(1.467,2.924)%
  --(1.519,2.867)--(1.570,2.809)--(1.622,2.752)--(1.673,2.695)--(1.725,2.637)--(1.776,2.580)%
  --(1.828,2.522)--(1.879,2.465)--(1.931,2.408)--(1.982,2.350)--(2.034,2.293)--(2.085,2.235)%
  --(2.137,2.178)--(2.188,2.121)--(2.240,2.063)--(2.291,2.006)--(2.343,1.949)--(2.394,1.891)%
  --(2.446,1.834)--(2.497,1.776)--(2.549,1.719)--(2.600,1.662)--(2.652,1.604)--(2.703,1.547)%
  --(2.755,1.489)--(2.806,1.432)--(2.858,1.375)--(2.909,1.317)--(2.961,1.260)--(3.012,1.202)%
  --(3.064,1.145)--(3.115,1.088)--(3.167,1.030)--(3.207,0.985);
\gpcolor{color=gp lt color border}
\gpsetdashtype{gp dt solid}
\gpsetlinewidth{1.00}
\draw[gp path] (0.952,4.263)--(0.952,0.985)--(6.051,0.985)--(6.051,4.263)--cycle;
\gpdefrectangularnode{gp plot 1}{\pgfpoint{0.952cm}{0.985cm}}{\pgfpoint{6.051cm}{4.263cm}}
\end{tikzpicture}
}
 \caption{Case 2}
   \end{subfigure}
  \qquad
  \begin{subfigure}[b]{0.21\textwidth}
  \centering
  \scalebox{.7}{\begin{tikzpicture}[gnuplot]
\tikzset{every node/.append style={font={\normalsize}}}
\path (0.000,0.000) rectangle (6.604,4.572);
\gpcolor{color=gp lt color border}
\gpsetlinetype{gp lt border}
\gpsetdashtype{gp dt solid}
\gpsetlinewidth{1.00}
\draw[gp path] (0.952,0.985)--(1.132,0.985);
\draw[gp path] (6.051,0.985)--(5.871,0.985);
\node[gp node right] at (0.768,0.985) {0};
\draw[gp path] (1.501,0.985)--(1.501,1.165);
\draw[gp path] (1.501,4.263)--(1.501,4.083);
\node[gp node center] at (1.501,0.677) {$\rho$};
\draw[gp path] (0.952,4.263)--(0.952,0.985)--(6.051,0.985)--(6.051,4.263)--cycle;
\gpsetdashtype{gp dt 2}
\draw[gp path](1.502,3.324)--(1.502,0.986);
\node[gp node center,rotate=-270] at (0.292,2.624) {velocity};
\node[gp node center] at (3.501,0.215) {time};
\node[gp node right] at (4.583,3.929) {SV};
\gpcolor{rgb color={0.580,0.000,0.827}}
\gpsetdashtype{gp dt 1}
\gpsetlinewidth{3.00}
\draw[gp path] (4.767,3.929)--(5.683,3.929);
\draw[gp path] (0.952,3.170)--(1.004,3.185)--(1.055,3.199)--(1.107,3.213)--(1.158,3.228)%
  --(1.210,3.242)--(1.261,3.256)--(1.313,3.271)--(1.364,3.285)--(1.416,3.299)--(1.467,3.314)%
  --(1.519,3.314)--(1.570,3.285)--(1.622,3.256)--(1.673,3.228)--(1.725,3.199)--(1.776,3.170)%
  --(1.828,3.141)--(1.879,3.113)--(1.931,3.084)--(1.982,3.055)--(2.034,3.027)--(2.085,2.998)%
  --(2.137,2.969)--(2.188,2.941)--(2.240,2.912)--(2.291,2.883)--(2.343,2.854)--(2.394,2.826)%
  --(2.446,2.797)--(2.497,2.768)--(2.549,2.740)--(2.600,2.711)--(2.652,2.682)--(2.703,2.654)%
  --(2.755,2.625)--(2.806,2.596)--(2.858,2.567)--(2.909,2.539)--(2.961,2.510)--(3.012,2.481)%
  --(3.064,2.453)--(3.115,2.424)--(3.167,2.395)--(3.218,2.367)--(3.270,2.338)--(3.321,2.309)%
  --(3.373,2.281)--(3.424,2.252)--(3.476,2.223)--(3.527,2.194)--(3.579,2.166)--(3.630,2.137)%
  --(3.682,2.108)--(3.733,2.080)--(3.785,2.051)--(3.836,2.022)--(3.888,1.994)--(3.939,1.965)%
  --(3.991,1.936)--(4.042,1.907)--(4.094,1.879)--(4.145,1.850)--(4.197,1.821)--(4.248,1.793)%
  --(4.300,1.764)--(4.351,1.735)--(4.403,1.707)--(4.454,1.678)--(4.506,1.649)--(4.557,1.621)%
  --(4.609,1.592)--(4.660,1.563)--(4.712,1.534)--(4.763,1.506)--(4.815,1.477)--(4.866,1.448)%
  --(4.918,1.420)--(4.969,1.391)--(5.021,1.362)--(5.072,1.334)--(5.124,1.305)--(5.175,1.276)%
  --(5.227,1.247)--(5.278,1.219)--(5.330,1.190)--(5.381,1.161)--(5.433,1.133)--(5.484,1.104)%
  --(5.536,1.075)--(5.587,1.047)--(5.639,1.018)--(5.690,0.989)--(5.697,0.985);
\gpcolor{color=gp lt color border}
\node[gp node right] at (4.583,3.621) {POV};
\gpcolor{rgb color={0.000,0.620,0.451}}
\draw[gp path] (4.767,3.621)--(5.683,3.621);
\draw[gp path] (0.952,4.208)--(1.004,4.151)--(1.055,4.094)--(1.107,4.036)--(1.158,3.979)%
  --(1.210,3.921)--(1.261,3.864)--(1.313,3.807)--(1.364,3.749)--(1.416,3.692)--(1.467,3.634)%
  --(1.519,3.577)--(1.570,3.520)--(1.622,3.462)--(1.673,3.405)--(1.725,3.347)--(1.776,3.290)%
  --(1.828,3.233)--(1.879,3.175)--(1.931,3.118)--(1.982,3.061)--(2.034,3.003)--(2.085,2.946)%
  --(2.137,2.888)--(2.188,2.831)--(2.240,2.774)--(2.291,2.716)--(2.343,2.659)--(2.394,2.601)%
  --(2.446,2.544)--(2.497,2.487)--(2.549,2.429)--(2.600,2.372)--(2.652,2.314)--(2.703,2.257)%
  --(2.755,2.200)--(2.806,2.142)--(2.858,2.085)--(2.909,2.027)--(2.961,1.970)--(3.012,1.913)%
  --(3.064,1.855)--(3.115,1.798)--(3.167,1.740)--(3.218,1.683)--(3.270,1.626)--(3.321,1.568)%
  --(3.373,1.511)--(3.424,1.454)--(3.476,1.396)--(3.527,1.339)--(3.579,1.281)--(3.630,1.224)%
  --(3.682,1.167)--(3.733,1.109)--(3.785,1.052)--(3.836,0.994)--(3.844,0.985);
\gpcolor{color=gp lt color border}
\gpsetdashtype{gp dt solid}
\gpsetlinewidth{1.00}
\draw[gp path] (0.952,4.263)--(0.952,0.985)--(6.051,0.985)--(6.051,4.263)--cycle;
\gpdefrectangularnode{gp plot 1}{\pgfpoint{0.952cm}{0.985cm}}{\pgfpoint{6.051cm}{4.263cm}}
\end{tikzpicture}
}
 \caption{Case 3}
   \end{subfigure}
  \qquad
  \begin{subfigure}[b]{0.21\textwidth}
  \centering
  \scalebox{.7}{\begin{tikzpicture}[gnuplot]
\tikzset{every node/.append style={font={\normalsize}}}
\path (0.000,0.000) rectangle (6.604,4.572);
\gpcolor{color=gp lt color border}
\gpsetlinetype{gp lt border}
\gpsetdashtype{gp dt solid}
\gpsetlinewidth{1.00}
\draw[gp path] (0.952,0.985)--(1.132,0.985);
\draw[gp path] (6.051,0.985)--(5.871,0.985);
\node[gp node right] at (0.768,0.985) {0};
\draw[gp path] (1.501,0.985)--(1.501,1.165);
\draw[gp path] (1.501,4.263)--(1.501,4.083);
\node[gp node center] at (1.501,0.677) {$\rho$};
\draw[gp path] (0.952,4.263)--(0.952,0.985)--(6.051,0.985)--(6.051,4.263)--cycle;
\gpsetdashtype{gp dt 2}
\draw[gp path](1.502,3.324)--(1.502,0.986);
\node[gp node center,rotate=-270] at (0.292,2.624) {velocity};
\node[gp node center] at (3.501,0.215) {time};
\node[gp node right] at (4.583,3.929) {SV};
\gpcolor{rgb color={0.580,0.000,0.827}}
\gpsetdashtype{gp dt 1}
\gpsetlinewidth{3.00}
\draw[gp path] (4.767,3.929)--(5.683,3.929);
\draw[gp path] (0.952,3.170)--(1.004,3.185)--(1.055,3.199)--(1.107,3.213)--(1.158,3.228)%
  --(1.210,3.242)--(1.261,3.256)--(1.313,3.271)--(1.364,3.285)--(1.416,3.299)--(1.467,3.314)%
  --(1.519,3.314)--(1.570,3.285)--(1.622,3.256)--(1.673,3.228)--(1.725,3.199)--(1.776,3.170)%
  --(1.828,3.141)--(1.879,3.113)--(1.931,3.084)--(1.982,3.055)--(2.034,3.027)--(2.085,2.998)%
  --(2.137,2.969)--(2.188,2.941)--(2.240,2.912)--(2.291,2.883)--(2.343,2.854)--(2.394,2.826)%
  --(2.446,2.797)--(2.497,2.768)--(2.549,2.740)--(2.600,2.711)--(2.652,2.682)--(2.703,2.654)%
  --(2.755,2.625)--(2.806,2.596)--(2.858,2.567)--(2.909,2.539)--(2.961,2.510)--(3.012,2.481)%
  --(3.064,2.453)--(3.115,2.424)--(3.167,2.395)--(3.218,2.367)--(3.270,2.338)--(3.321,2.309)%
  --(3.373,2.281)--(3.424,2.252)--(3.476,2.223)--(3.527,2.194)--(3.579,2.166)--(3.630,2.137)%
  --(3.682,2.108)--(3.733,2.080)--(3.785,2.051)--(3.836,2.022)--(3.888,1.994)--(3.939,1.965)%
  --(3.991,1.936)--(4.042,1.907)--(4.094,1.879)--(4.145,1.850)--(4.197,1.821)--(4.248,1.793)%
  --(4.300,1.764)--(4.351,1.735)--(4.403,1.707)--(4.454,1.678)--(4.506,1.649)--(4.557,1.621)%
  --(4.609,1.592)--(4.660,1.563)--(4.712,1.534)--(4.763,1.506)--(4.815,1.477)--(4.866,1.448)%
  --(4.918,1.420)--(4.969,1.391)--(5.021,1.362)--(5.072,1.334)--(5.124,1.305)--(5.175,1.276)%
  --(5.227,1.247)--(5.278,1.219)--(5.330,1.190)--(5.381,1.161)--(5.433,1.133)--(5.484,1.104)%
  --(5.536,1.075)--(5.587,1.047)--(5.639,1.018)--(5.690,0.989)--(5.697,0.985);
\gpcolor{color=gp lt color border}
\node[gp node right] at (4.583,3.621) {POV};
\gpcolor{rgb color={0.000,0.620,0.451}}
\draw[gp path] (4.767,3.621)--(5.683,3.621);
\draw[gp path] (3.403,4.263)--(3.424,4.240)--(3.476,4.182)--(3.527,4.125)--(3.579,4.068)%
  --(3.630,4.010)--(3.682,3.953)--(3.733,3.895)--(3.785,3.838)--(3.836,3.781)--(3.888,3.723)%
  --(3.939,3.666)--(3.991,3.609)--(4.042,3.551)--(4.094,3.494)--(4.145,3.436)--(4.197,3.379)%
  --(4.248,3.322)--(4.300,3.264)--(4.351,3.207)--(4.403,3.149)--(4.454,3.092)--(4.506,3.035)%
  --(4.557,2.977)--(4.609,2.920)--(4.660,2.862)--(4.712,2.805)--(4.763,2.748)--(4.815,2.690)%
  --(4.866,2.633)--(4.918,2.575)--(4.969,2.518)--(5.021,2.461)--(5.072,2.403)--(5.124,2.346)%
  --(5.175,2.288)--(5.227,2.231)--(5.278,2.174)--(5.330,2.116)--(5.381,2.059)--(5.433,2.002)%
  --(5.484,1.944)--(5.536,1.887)--(5.587,1.829)--(5.639,1.772)--(5.690,1.715)--(5.742,1.657)%
  --(5.793,1.600)--(5.845,1.542)--(5.896,1.485)--(5.948,1.428)--(5.999,1.370)--(6.051,1.313);
\gpcolor{color=gp lt color border}
\gpsetdashtype{gp dt solid}
\gpsetlinewidth{1.00}
\draw[gp path] (0.952,4.263)--(0.952,0.985)--(6.051,0.985)--(6.051,4.263)--cycle;
\gpdefrectangularnode{gp plot 1}{\pgfpoint{0.952cm}{0.985cm}}{\pgfpoint{6.051cm}{4.263cm}}
\end{tikzpicture}
}
 \caption{Case 4}
   \end{subfigure}
\caption{the velocities of the \ac{sv} and the \ac{pov} in Theorem~\ref{thm:RSSExSafetyThm}. Note that the assumption $0<a_\mathrm{brake,min}<a_\mathrm{brake,max}$ is crucial here. }
\label{fig:RSSExVelocityPatterns}
\end{figure}
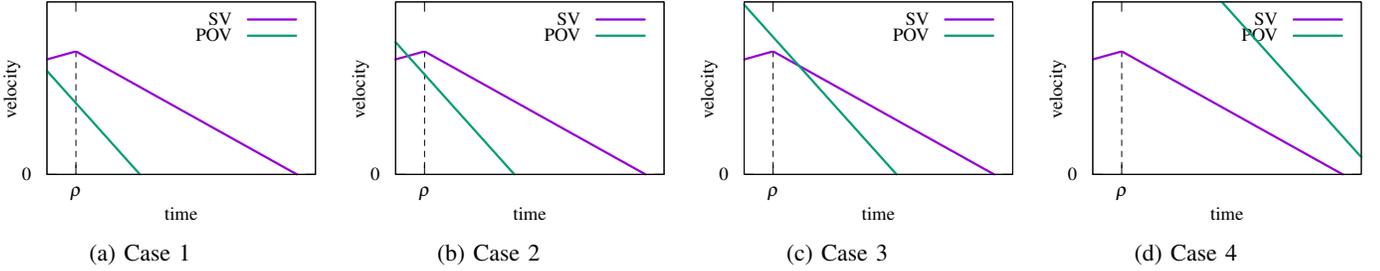

\subsection{The Responsibility Proof}
The following is not a ``theorem'' in a rigorous sense---this is because the RSS responsibility principles (Definition~\ref{def:RSSPrinciples}) are not formally defined conditions. We nevertheless provide an argument for it; such arguments are of great practical values, when it comes to such matters as explainability to the public and attribution of liability. 

\begin{theorem}[responsibility of $R_{\slsd}$]\label{thm:RSSExResponsibilityThm}
 Assume the same setting as in Theorem~\ref{thm:RSSExSafetyThm}. Then the execution of the proper response $\PR_{\slsd}$ satisfies the RSS responsibility principles. 
\end{theorem}

\begin{proof}(an informal argument)
 The principles~1 and~5 are ensured by safety (Theorem~\ref{thm:RSSExSafetyThm}). The other principles (2--4) do not apply to the current driving scenario $S_{\slsd}$. 
\end{proof}

\section{Usages of RSS}
\label{subsec:RSSUsages}
Some usages of the RSS framework (Figure~\ref{fig:RSSFrmwk}) have been already hinted in the above. Here we go into their details,  providing some pointers to related scientific studies and ongoing practical efforts at the same time.

\subsection{Attribution of Liability}
\label{subsubsec:RSSAttribLiab}
 Attribution of liability is one of the first applications that were pursued using RSS. We have already discussed it briefly in Section~\ref{subsec:RSSFrmwk}: when a collision happened, the traffic participant that did not comply with the RSS responsibility principles is held liable. The validity of this reasoning is supported by the mathematical elements of RSS: the RSS safety rules, especially their safety and responsibility requirements, ensure that there is no collision as long as all the traffic participants adhere to the RSS responsibility principles. 

Here, the roles of RSS safety rules are twofold. For one, an RSS safety rule can be seen as a mathematical incarnation of the RSS responsibility principles, one that enables a rigorous safety proof (the safety theorems in Figure~\ref{fig:RSSFrmwk}). Another important role is as an \emph{evidence of acting responsibly}. The RSS responsibility principles are informal conditions; in contrast, RSS safety rules are rigorous and mathematically checkable. Demonstration of the compliance with RSS safety rules  is therefore  a strong evidence of acting responsibly, via the responsibility requirements for those rules (the responsibility theorems in Figure~\ref{fig:RSSFrmwk}). 

A comprehensive case study of such use of RSS is found in~\cite{ShashuaSS18NHTSA}, where the RSS framework is applied to NHTSA pre-crash scenarios. 




\subsection{As a Safety Metric}
\label{subsubsec:RSSSafetyMetric}
\index{safety metric}
Safety metrics for \ac{ads} are an active topic of research, see e.g.~\cite{SilberlingWAKL20,WangLS20RealisticSingleShotLongTerm,AltekarEWCWCRJ20,ZhaoZMSLLZ_IV20,WengRDSB20IV}. The question here is how to evaluate the safety of a trajectory of a vehicle. A naive metric can be given by the minimum distance---the minimum distance to obstacles and  other traffic participants  exhibited during the trajectory---but its limitations are obvious, too, such as the insensitivity to the vehicle speed. A safety metric that has been commonly used is the \dfn{time to collision (TTC)}. 

It is straightforward to derive another safety metric from the last usage of RSS (Section~\ref{subsubsec:RSSAttribLiab}). This \emph{RSS safety metric} measures the degree with which a trajectory satisfies RSS safety rules: the margin with which the inequalities in the RSS safety conditions are satisfied is its safety score. Violation of the RSS safety conditions results in a negative safety score.

Multiple variations and extensions are possible for the above (informal and rudimentary) notion of RSS safety metric. Several of such are presented in~\cite{AltekarEWCWCRJ20}; among them are a metric that takes into account whether a proper response was engaged when needed.




\subsection{Formal Verification of ADS Safety}
\label{subsubsec:RSSFormalVefif}

 Another obvious usage of RSS is for \emph{formal verification} of \ac{ads} safety.
By proving that a vehicle's control complies with RSS safety rules, one can conclude the safety of the vehicle via the safety theorems of the RSS safety rules. Here, the ``compliance with RSS safety rules'' means the following.

\begin{definition}[compliance with RSS safety rules]
\label{def:RSSComplianceWithRules}
\index{compliance with RSS safety rules}
 Let $\mathcal{R}=(R_{1}, R_{2},\dotsc, R_{n})$ be a list of RSS safety rules, where $R_{i}=(C_{i}, \PR_{i})$ for each $i\in \{1, 2, \dotsc, n\}$. We say that a vehicle's trajectory 
\emph{complies with $\mathcal{R}$} if, at each moment
in the trajectory, either
 \begin{itemize}
  \item the RSS safety condition $C_{i}$ is true for some $i\in \{1, 2, \dotsc, n\}$, or
  \item the proper response $\PR_{i}$ is being executed for some $i\in \{1, 2, \dotsc, n\}$. Additionally, we require that this execution of $\PR_{i}$  started in a state in which the corresponding RSS safety condition $C_{i}$ was true.
 \end{itemize}
\end{definition}

\begin{theorem}[safety of an RSS-compliant trajectory]\label{thm:RSSComplianceImpliesSafety}
 Let  $\mathcal{R}=(R_{1}, R_{2},\dotsc, R_{n})$ be a list of RSS safety rules, where $R_{i}=(C_{i}, \PR_{i})$ for each $i\in \{1, 2, \dotsc, n\}$. Assume further that each rule $R_{i}$ satisfies the safety requirement (Requirement~\ref{req:RSSSafetyAssurance}). If a vehicle's trajectory $T$ complies with $\mathcal{R}$ (Definition~\ref{def:RSSComplianceWithRules}), then $T$ exhibits no collision. 
\end{theorem}
\begin{proof}
 Consider an arbitrary moment of the trajectory $T$. If some RSS safety condition $C_{i}$ is true at that moment, then there is no collision at that moment---this follows from the remark immediately after Requirement~\ref{req:RSSSafetyAssurance}. If some proper response $\PR_{i}$ is being executed at that moment, the safety guaranteed in Requirement~\ref{req:RSSSafetyAssurance} implies that no collision occurs then. 
\end{proof}

Towards the goal of mathematically proving safety of an \ac{ads}, 
it suffices to ensure that every trajectory of the \ac{ads} complies with some list of RSS safety rules in the sense of Definition~\ref{def:RSSComplianceWithRules}. 
This is the consequence of Theorem~\ref{thm:RSSComplianceImpliesSafety}.
However, there are two major challenges in ensuring the above. 
\begin{itemize}
 \item An \ac{ads} controller is a complex system, involving a number of numeric optimization algorithms and statistical machine learning. Therefore, proving its compliance with RSS safety rules is hard. 
 \item Compliance in the sense of Definition~\ref{def:RSSComplianceWithRules} requires a comprehensive list $\mathcal{R}$ of RSS safety rules---otherwise there will be moments that are not covered by any $C_{i}$ or $\PR_{i}$. Obtaining such a comprehensive list will take an enormous effort: a safety proof for a single safety rule is nontrivial already for a simple scenario (Section~\ref{subsec:RSSEx}); and the number of driving scenarios to be covered is huge. 

\end{itemize}



\subsection{Safety Architecture}
\label{subsubsec:RSSSafetyArch}
A promising ``workaround'' to the last two challenges in formal verification (Section~\ref{subsubsec:RSSFormalVefif}) is the use of RSS in a \emph{safety architecture}. \index{safety architecture}

\begin{figure}[tbp]\centering
\includegraphics[bb=256 199 717 421,clip,width=25em]{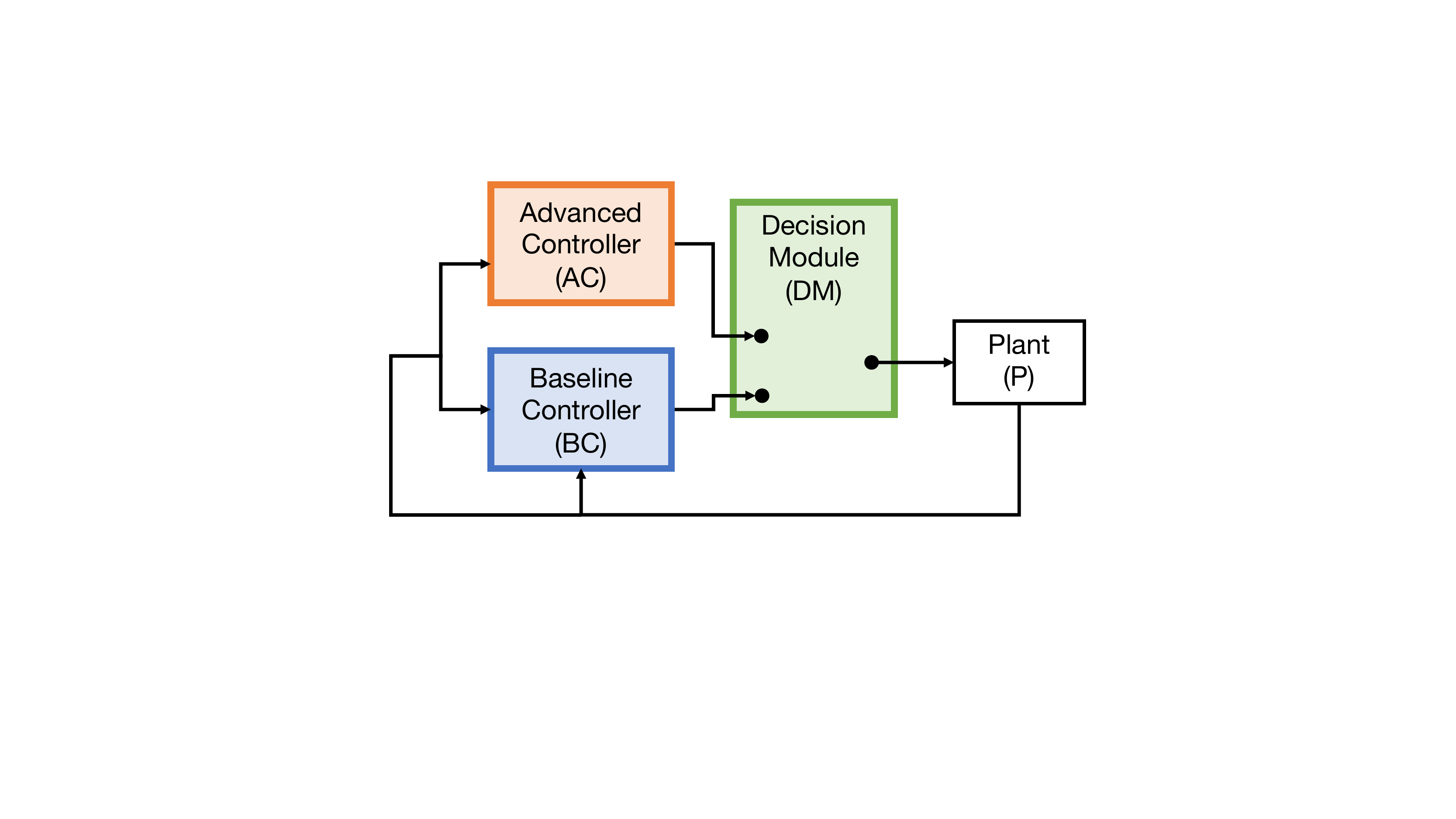}
\caption{the simplex architecture}
\label{fig:simplex}
\end{figure}

A prototypical safety architecture is the \emph{simplex architecture} shown in
\cref{fig:simplex}~\cite{CrenshowGRSK07,SetoKSC98}. Here, the \emph{advanced
  controller} (AC) is a complex controller that pursues not only safety but
also performance (such as comfort, progress, and fuel efficiency); the
\emph{baseline controller} (BC) is a simpler controller with a strong emphasis on
safety; and the \emph{decision module} (DM) switches between the two
controllers. The DM tries to use the AC as often as possible for its superior
performance. However, when the DM finds that the current situation is safety
critical, it switches to the BC whose behaviors are more predictable and easier
to analyze. \index{simplex architecture}

A notable feature of safety architectures such as the simplex architecture is that they enable formal verification of a system that contains a black-box component (such as the AC in the simplex architecture). It does so by ``wrapping'' the black-box component with safety-centric components that closely monitor the black-box component and overtake  control when it is needed. 

The elements of RSS map naturally to the simplex architecture:
\begin{itemize}
  \item The DM makes use of  RSS safety conditions.  When they are about to be violated, the DM switches the control from the AC to the BC.
  \item The BC implements  proper responses. It thus executes a control strategy whose safety is guaranteed.
  \item If it so happens that the BC restores the RSS safety condition, then the DM switches the control back to  the AC.
\end{itemize}
The way the resulting \emph{RSS-supervised controller} operates reflects the argument in Section~\ref{subsubsec:RSSFormalVefif}: it lets the AC control as long as some RSS safety condition $C_{i}$ is true; once it comes close to violating all the RSS safety conditions, it executes a proper response $\PR_{i}$ to maintain safety. The possibility of switching back to the AC is advantageous for performance. 

Such use of RSS in a safety architecture successfully addresses the two challenges in formal verification (Section~\ref{subsubsec:RSSFormalVefif}). 
\begin{itemize}
 \item 
 The complexity of the \ac{ads} controller becomes no problem since it is confined to the AC---the safety of the whole  architecture is ensured exclusively by the properties of the DM and the BC. 
 \item 
On the difficulty of obtaining a comprehensive list of RSS safety rules, a safety architecture offers a \emph{proactive} and \emph{best-effort} alternative to formal verification. Even if a rule list $\mathcal{R}=(R_{1}, R_{2},\dotsc, R_{n})$ does not cover all possible driving scenarios, $\mathcal{R}$ can still be implemented in a safety architecture and ensure the safety of those scenarios which are covered by $\mathcal{R}$. (For those driving scenarios which are not covered, we just let the AC do its best to maintain safety). After all,  formal verification is about analyzing a given controller (Section~\ref{subsubsec:RSSFormalVefif}), while safety architectures are about modifying the controller (identified with the AC) and making it safe. 
\end{itemize}

Such use of RSS in a safety architecture is advocated increasingly often. 
In fact,
some recent papers such as~\cite[Fig.~1]{OborilS20IV} present RSS in the format
of safety architectures, although they may not explicitly refer to the term ``safety
architecture.''



\section{Current State and Future Directions}
\label{subsec:RSSFuture}
\subsection{Making a Rule Set Comprehensive}
\label{subsubsec:RSSFWComprehensive}
Towards the ultimate goal of ADS safety proofs, one important area of RSS that requires further work is the formulation of RSS rules. As we already discussed, RSS rules are formulated and proved correct in a scenario-specific manner, and there are a huge number of driving scenarios. 

While a comprehensive RSS rule set is obviously important and desired, the path to it  might seem endless. 
We nevertheless believe that  efforts in this direction are worthwhile. 
\begin{itemize}
 \item 
 One reason is the use of the rules in a safety architecture (Section~\ref{subsubsec:RSSSafetyArch}). Even if the rule list $\mathcal{R}$ at hand is not comprehensive (and hence not enough to prove the overall safety), the ``best-effort'' usage in a safety architecture ensures safety at least for the scenarios that are already covered by  $\mathcal{R}$. 
 \item 
 Another reason is that  RSS rules are irrevocable, with their correctness being mathematically established. Once derived, they can be used for the coming dozens of years or even longer---they can be seen as common assets of humankind. With more efforts thrown in for the formulation of RSS rules, the rule set grows monotonically. 
\end{itemize}

\subsection{Tool Support for Using, Deriving, and Verifying RSS Rules}
\label{subsubsec:RSSFWTools}
Tool support for RSS has been pursued actively in recent years. One example is an implementation of some RSS rules: it is offered as a library~\cite{GassmannOBLYEAA19IV} that can be used in combination with simulation environments such as Baidu Apollo. This implementation is for \emph{using} RSS rules that have been already derived and verified.

Another area that calls for tool support is the derivation of RSS rules. As we discussed in Section~\ref{subsubsec:RSSSafetyArch}, RSS rule derivation requires systematic and organized efforts, much like formal verification by theorem proving. Therefore the task needs tool support---much like tool support for theorem proving is given by \emph{proof assistants} such as Coq~\cite{Coq} and Isabelle/HOL~\cite{nipkow2002isabelle}---otherwise the efforts will be marred with human errors. \index{proof assistant}

Yet another area that calls for tool support is the (formal) verification of RSS rules. Existing correctness proofs for RSS rules (such as the ones in Section~\ref{subsec:RSSEx}) are mathematical yet not formalized or mechanized; it is not hard to imagine human errors in such proofs, especially when working with complex driving scenarios. The problem of giving formalized and mechanized proofs to RSS rules is investigated in~\cite{RoohiKWSL18arxiv}. Their trial is based on a rigorous notion of signal; they argue that none of the existing \emph{automated} verification tools is suited for the verification problem. We believe that the use of theorem provers that allow human interaction should be pursued. An example of such provers is KeYmaera~X~\cite{FultonMQVP15}.

Note that the last two directions (derivation and verification) are the motivation of the current introduction (namely building a logical theory and formalization). We believe that our identification of
\begin{itemize}
 \item 
  two requirements (\cref{req:RSSSafetyAssurance,req:RSSResponsibility}) and 
 \item 
their assume-guarantee relationship (\cref{fig:RSSFrmwk}) 
\end{itemize}
in this paper
is a necessary first step in the two directions.

\subsection{Permissive, and Thus Practical, RSS Rules}
The scope of RSS rules is principally safety. They can sacrifice other practical performance metrics such as comfort, progress, and fuel efficiency. For greater practical utility, it is desired that RSS rules are less conservative and more permissive. 

One work in this direction is~\cite{RSSplusIV2021accepted}; it proposes an extension of RSS in which the balance between safety and other performance metrics is proactively adjusted. Other promising directions seem to include the following. 

\begin{itemize}
 \item \emph{Fine-grained formalization} of the RSS responsibility principles, so that the formalization respects different  traffic circumstances (urban, rural, or highway),  different regions and countries, and different driving cultures.

 \item \emph{RSS that is aware of intentions and knowledge}. Responsibilities of other vehicles will be totally different depending on whether the \ac{sv}'s turn signal is blinking or not. Moreover, when \ac{sv} makes a maneuver with its turn signal blinking, it is also desired that the \ac{sv} makes sure that other vehicles are aware of the turn signal. 
 \item \emph{Enlarging the variety of proper responses}. For example, \cite{deIacoSC20IV} proposes an extension of the RSS rule in Section~\ref{subsec:RSSEx} in which swerves (in addition to braking) is allowed as evasive maneuvers~\cite{deIacoSC20IV}. This makes the corresponding safety condition much weaker, and the RSS rule more permissive. 
\end{itemize}

\subsection{Perceptual Uncertainties}
\index{perceptual uncertainty}
In the presence of perceptual uncertainties (such as  errors in position measurement and object recognition), the perceived values that RSS safety conditions depend on can be erroneous. 
Making RSS rules tolerant of such perceptual uncertainties is pursued in~\cite{SalayCEASW20PURSS}; the methodology used there is more generally formulated in~\cite{DBLP:conf/nfm/KobayashiSHCIK21} and used in combination with the modeling and verification framework Event-B~\cite{AbrialBHHMV10}. 

A big source of perceptual uncertainties in modern \ac{ads} is its \emph{statistical machine learning} components (neural networks to be specific). 
One way to lessen such ML-related uncertainties is to enrich the  output of an ML component. For example, the use of DNNs' confidence scores is proposed in~\cite{AngusCS19arxiv}. In~\cite{ChowRWGJALKC20}, it is proposed  to look at inconsistencies between perceptual data of different modes.


\section*{Acknowledgment}
Thanks are due to 
  Clovis Eberhart,
  James Haydon,
  J\'{e}r\'{e}my Dubut,
  Rose Borher,
  Tsutomu Kobayashi,
  Sasinee Pruekprasert,
  Ahmet Cetinkaya,
  Xiaoyi Zhang, and 
  Akihisa Yamada for discussions. 
The author is supported by ERATO HASUO Metamathematics for Systems Design Project (No.\ JPMJER1603), JST.

\bibliographystyle{plain}
\bibliography{decisionAD}

\end{document}